\documentclass[letterpaper, 10 pt, conference]{ieeeconf}  

\IEEEoverridecommandlockouts                              

\usepackage{blindtext, graphicx}
\usepackage{amsmath}
\usepackage{url}
\usepackage{fancyhdr}
\usepackage{latexsym}
\usepackage{amssymb}
\usepackage{color}
\usepackage{comment}
\newtheorem{problem}{Problem}
\newtheorem{mydef}{Definition}
\newtheorem{theorem}{Theorem}

\newtheorem{example}{Example}

\newtheorem{assertion}{Assertion}

\usepackage{tikz}
\usepackage{lipsum}
\usepackage{mathtools}
\usepackage{cuted}
\usepackage{tikz}
\usetikzlibrary{arrows,positioning,chains,fit,shapes,automata} 
\usepackage{algorithm}
\usepackage{algorithmic}

\title{\LARGE \bf Vector Autoregressive POMDP Model Learning and Planning for Human-Robot Collaboration}
\author{Wei Zheng and Hai Lin
\thanks{The partial support of the National Science Foundation (Grant No. CNS-1446288, ECCS-1253488, IIS-1724070) and of the Army Research Laboratory (Grant No. W911NF- 17-1-0072) is gratefully acknowledged.}
\thanks{Wei Zheng and Hai Lin are with the Department of Electrical Engineering, University of Notre Dame, Notre Dame, IN, 46556 USA. {\tt\small wzheng1@nd.edu, hlin1@nd.edu}.}
}

\begin{document}

\maketitle
\thispagestyle{empty}
\pagestyle{empty}

\begin{abstract}
Human-robot collaboration (HRC) has emerged as a hot research area at the intersection of control, robotics, and psychology in recent years. It is of critical importance to obtain an expressive but meanwhile tractable model for human beings in HRC. In this paper, we propose a model called Vector Autoregressive POMDP (VAR-POMDP) model which is an extension of the traditional POMDP model by considering the correlation among observations. {The VAR-POMDP model is more powerful in the expressiveness of features than the traditional continuous observation POMDP since the traditional one is a special case of the VAR-POMDP model.} Meanwhile, the proposed VAR-POMDP model is also tractable, as we show that it can be effectively learned from data and we can extend point-based value iteration (PBVI) to VAR-POMDP planning. Particularly, in this paper, we propose to use the Bayesian non-parametric learning to decide potential human states and learn a VAR-POMDP model using data collected from human demonstrations. Then, we consider planning with respect to PCTL which is widely used as safety and reachability requirement in robotics. Finally, the advantage of using the proposed model for HRC is validated by experimental results using data collected from a driver-assistance test-bed. 


\end{abstract}

\section{Introduction}
Human-robot collaboration (HRC) studies how to achieve effective collaborations between human and robots to synthetically combine the strengths of human beings and robots. While robots have advantages in handling repetitive tasks with high precision and long endurance, human beings are much more flexible to changing factors or uncertain environments that are difficult for robots to adapt. Therefore, to establish an efficient collaboration between human and robots is the core problem in the design of the HRC system. 

To achieve an effective HRC, it is of critical importance to obtain an expressive but meanwhile tractable model for HRC. Among several types of models in the HRC literature, such as the ACT-R/E model \cite{trafton2013act}, the IDDM Model \cite{wang2013probabilistic} and the TLP model \cite{wang2018human}, the POMDP model has emerged as a popular choice in recent years \cite{nikolaidis2015efficient}\cite{zhang2019performance}. As a general probabilistic system model to capture uncertainties from actuation errors, sensing noises and human behaviors, the POMDP model provides a comprehensive framework for the system modeling and sequential decision making. 
Most of the existing results assume that the POMDP model is given \cite{broz2011designing}\cite{gopalan2015modeling} or the number of states is given \cite{jaulmes2005active}\cite{ross2008bayes} before learning the model from data. However, the state space could be tedious to be predefined and the number of hidden states could be case dependent especially when human is involved. In our previous work \cite{zheng2018pomdp}, we dropped these assumptions and proposed a Bayesian non-parametric learning approach to infer the structure of the POMDP model, such as the number of states, from data. However, hidden states of the POMDP model being learned can only model static properties of the observed data. For example, in the driving scenario, each state is a cluster of positions of the human hand. Then the human intention, say turning right, can be inferred if the observed position belongs to the cluster. Dynamic properties such as turning speed and acceleration could not be modeled and distinguished. This is because we did not consider the correlations among observations in the POMDP model.

In order to fill this gap, we propose a new type of model for HRC, called Vector Autoregressive POMDP (VAR-POMDP) model, which takes the observation correlation into consideration and hence extends the existing POMDP model. Our main objective in this paper is to show that the proposed VAR-POMDP model can achieve a good trade-off between model expressiveness and tractability.

The expressiveness of the proposed model is clear as the POMDP model becomes a special case of our proposed model. To illustrate the tractability of the proposed model, we investigate both the model learning and planning issues in this paper. First, in the model learning process, we do not assume the state space is given or the bound on the number of states is known. Our basic idea is to use a Bayesian non-parametric learning method to automatically identify the number of hidden states. Secondly, in the planning process, we consider the probabilistic computation tree logic (PCTL) as a formal specification since it is widely used as a safety and reachability requirement in robotic application \cite{zhang2015learning}. The PCTL bounded until model checking problem can be converted to a finite horizon dynamic programming problem. We show that the value function can be approximated by a piece-wise linear function by extending the PBVI on the VAR-POMDP model. The effectiveness of the proposed learning and planning algorithms are illustrated in real experiments.

The main contribution of this paper is twofold. First, the VAR-POMDP which considers the correlation of observations is proposed to model the HRC process and the corresponding learning framework is proposed to learn the VAR-POMDP model from demonstrations using the Bayesian non-parametric learning method. Secondly, the PBVI algorithm is extended to the VAR-POMDP model to solve a finite step dynamic programming problem and therefore the bounded until model checking problem.

The rest of the paper is organized as follows. Section \ref{sec:preliminaries} presents the formal definition of the VAR-POMDP model and formulates the problem. The learning framework with corresponding experiment results are shown in Section \ref{sec:mainresults}. Section \ref{sec:problem_formulation} shows how to extend the PBVI algorithm on the VAR-POMDP model and section \ref{sec:conc} concludes the paper. 

\section{VAR-POMDP Model}\label{sec:preliminaries}

The VAR-POMDP model is inspired from the autoregressive hidden Markov model (AR-HMM) which is an extension of the HMM. Distinct from the HMM which assumes the independence of observations, the autoregressive model specifies that the observed variable depends linearly on its previous values with certain uncertainties. This correlation property exhibits in the behavior of human motion \cite{abe2007study}. Thus, we consider the correlation of observations in the POMDP model and propose the VAR-POMDP model. 

\begin{mydef}\label{def:pomdp}
The VAR-POMDP model is defined as a tuple $P= (S, A, O, T, E, R, L)$
where 
\begin{itemize}
    \item $S$ is a finite set of states.
    \item $A$ is a finite set of decision actions.
    \item $O$ is a set of continuous observations. 
    \item $T :S \times A \times S \to [0, 1]$ is a transition function which defines the probability over the next state $s'$ after taking an action $a$ from the state $s$.
    \item $E : \mathbf{O} \times S \times O \to \mathbb{R}$ is an observation function which defines the distribution over the observation $o'$ that may occur in state $s'$ conditional on observation history $\mathbf{o}$.  
    \item $R: S \times A \to \mathbb{R} $ is a reward function. 
    \item $L: S \to 2^{AP}$ is a labelling function that assigns a subset of atomic propositions $AP$ to each state $s$.  
\end{itemize}
\end{mydef}
The difference between the VAR-POMDP model and the traditional POMDP model is the observation function. In traditional POMDP model, the observation function only dependents on the hidden state and the current observation. While in VAR-POMDP model, the observation function also relies on the observation history, namely, 
\begin{equation}
    o_t = \sum_{i=1}^r A_{i,s_t} o_{t-i} + e(s_t), 
\vspace{-1mm}    
\end{equation}
where $e(s_t) \sim \mathcal{N}(0,\Sigma_{s_t})$ are Gaussian noise modeling the uncertainty, matrices $\{A_{1,s_t},...,A_{r,s_t}\}$ are lag matrices under mode $s_t$. Note that the continuous observation POMDP with Gaussian emissions is a special case of this model when lag matrices are zero and the Gaussian mean is replaced with a constant vector. Instead of using a constant value to characterize the motion feature, the VAR-POMDP uses lag matrices and the covariance matrix to characterize the feature. {Thus, dynamic features can be expressed and identified by the VAR-POMDP.} 


\subsection{Probabilistic Computation Tree Logic}
To achieve a performance-guaranteed model learning and planning framework, we use the PCTL as the formal specification to guide the designing process. PCTL is a probabilistic extension of computation tree logic which allows for probabilistic quantification of described properties \cite{hansson1994logic}. The syntax of PCTL is as follows: 
\begin{equation}
\begin{array}{ll}
    \phi:: & = true \mid AP \mid \neg \phi \mid \phi_1 \land \phi_2  \mid P_{\bowtie p} \psi, \\
    \psi:: & = X \phi \mid \phi_1 \cup^{\leq k} \phi_2 \mid \phi_1 \cup \phi_2,
\end{array}
\vspace{-1mm}
\end{equation}
where $AP$ is an atomic proposition, $\bowtie \in \{\leq,<,\geq, >\}$, $p \in [0,1]$, $k \in \mathbb{N}$ and $\cup^{\leq k}$ stands for bounded until. The soft deadline property make PCTL a widely used specification language for probabilistic model checkers.


In this paper, we consider the problem of VAR-POMDP model learning and planning for the bounded until specification in PCTL. 

\begin{problem}
\label{pro:problem1}
Given training data collected from an HRC process, learn a VAR-POMDP model $P$. Based on the learned model $P$, together with a given initial belief $b_0$ over the state, a finite horizon $H$ and a PCTL bounded until specification,
\begin{equation}
\label{spec:bounded}
          P_{\leq p} [\phi_1 \cup^{\leq H} \phi_2],
\vspace{-1mm}
\end{equation}
check whether or not the specification is satisfied. 
\end{problem}

The PCTL specification specifies the upper bound of the probability that, in a finite $H$ step, there are some states along a path making $\phi_2$ holds and $\phi_1$ holds in all states prior to that state. For example, if $\phi_1 = {true}$ and $\phi_2 = {'Fail'}$, the specification bounds the probability of the system going to states that cause the system failure. The specification in equation (\ref{spec:bounded}) is satisfied if and only if $p_{b_0}^{max}(\phi_1 \cup^{\leq H} \phi_2) \leq p$ where $p_{b_0}^{max}(\phi_1 \cup^{\leq H} \phi_2)$ is the maximum satisfaction probability with respect to belief $b_0$. The model checking problem is converted to a finite step optimization problem.

On one hand, the PCTL specification gives a performance requirement of the system which guides the model learning and controller designing process. On the other hand, using PCTL as specification avoids further defining the reward function for the VAR-POMDP model. Although algorithms such as inverse reinforcement learning can be used to recover the reward function \cite{choi2011inverse}, it is hard to explain the physical meaning of the reward. For PCTL specification, the reward can be clearly explained as satisfaction probability. 

\section{VAR-POMDP Model Learning}\label{sec:mainresults}


The proposed framework to learn the VAR-POMDP model is shown in Figure \ref{fig:Overview}. The action set is assumed given since it represents the capability of the robot. Using the training data collected from the HRC system, the state space of the model is identified using the Bayesian non-parametric learning method. The whole state space could be the product of the state space of human, robots and the environment. Based on the identified state space, transition probability and observation distribution can be learned from data.



\begin{figure}[!t]
\centering
\includegraphics[width=0.9\linewidth]{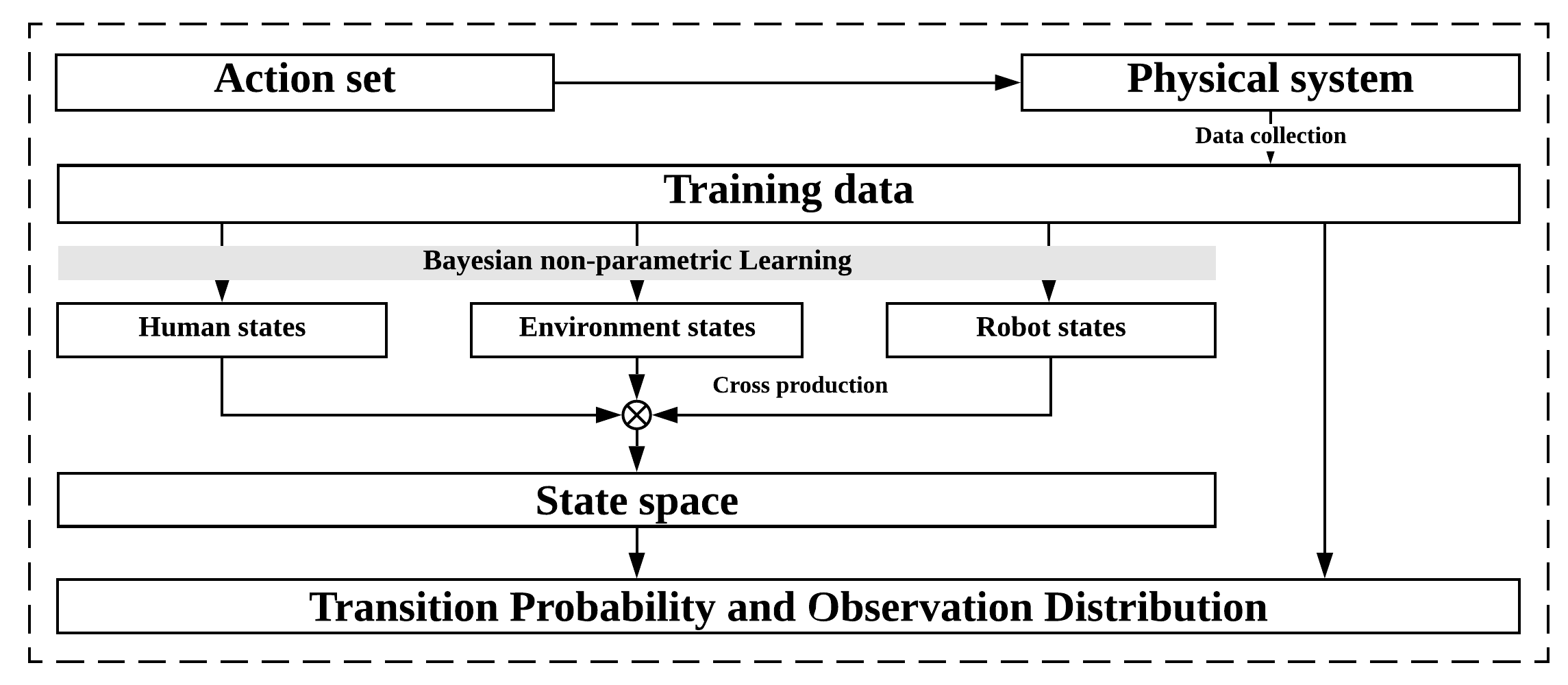}
\vspace{-3mm}
\caption{Overview of the proposed learning framework.}
\label{fig:Overview}
\vspace{-7mm}
\end{figure}

\subsection{Motion Feature Extraction}
\label{sec:MotionFE}
Instead of assuming the state space is given, we use the Bayesian non-parametric learning method to directly infer the state space from data. The training data consists of several $n$-dimensional time series which could be human motion trajectories collected from demonstrations. Taking advantages of the BP-AR-HMM framework proposed in \cite{fox2009sharing}, we do not only learn features of the data but also the number of features in a full Bayesian rule.

In the Bayesian non-parametric learning method, an AR-HMM is used as a generative model to model the relationship between the hidden features and observations. For each observed time series $Y^i=[y^i_{1},...,y^i_{T_i}]$ where $y_t^i \in \mathbb{R}^{n}$, we assume it is generated from the following model, 
\begin{equation}
\begin{aligned}
   z_t^{i} & \sim \pi_{z_{t-1}^{i}}^{i},\\
   y_t^{i} & = \sum_{j=1}^{r} A_{j,z_t^{i}} y_{t-j}^{i}+e_t^{i}(z_t^{i}),
    \end{aligned}
    \vspace{-2mm}
\end{equation}
where $e_t^{i}(k) \in \mathbb{R}^n$ is a zero-mean Gaussian noise with co-variance matrix $\Sigma_k$ to capture the uncertainty, and $r$ is the order of the autoregressive process. The variable $z^i_t \in \mathbb R$ is the hidden state and $\pi^i_j$ specifies the transition property for the state $j$. Each hidden state $z_t^i$ is characterized by a set of parameters $\theta_{z_t^i}=\{A_{1,z_t^{i}},...,A_{r,z_t^{i}},\Sigma_{z_t^{i}}\}$ where the parameters are used to characterize the corresponding features of the data. For example, in the scenario of using human motion trajectories as training data, the parameter $\{A_{1,z_t^{i}},...,A_{r,z_t^{i}}\}$ describes the dependence of the current position on historical positions. The physical meaning of features are motion patterns. 

{Compared with the HMM assumption which ignores the dependence among observations \cite{zheng2018pomdp}, the AR-HMM assumption makes the model be able to extract dynamic properties of the observing data since it considers the observing data as a dynamic system. Thus using the AR-HMM as the generative model is necessary when one cares about dynamic properties rather than static properties of the observing data.}

The traditional approach usually assumes the number of hidden states or the upper bound of hidden states is given. However, to get this prior knowledge is nontrivial especially when human is involved. In Bayesian non-parametric learning, a prior distribution is used and the number of hidden states can be inferred from data automatically. 
The BP-AR-HMM uses a Beta Process (BP) to generate a collection of an infinite number of points and assign each point a weight which is a flip coin probability. Then a Bernoulli Process (BeP) selects points that exhibit in each training data $Y^i$. These points are bond with the hidden states $z^i_t$ and therefore the feature parameters $\theta_{z_t^i}$. The Beta-Bernoulli Process together is used to model the correlation among time series. This process is summarized as follows,
\begin{equation}
\begin{aligned}
   B|B_0 & \sim \text{BP}(c,B_0),\\
   X_i|B & \sim \text{BeP}(B),\\
   \pi_j^{i}| f_i,\gamma, \kappa & \sim \text{Dir}([\gamma,...,\gamma+\kappa,\gamma,...] \otimes f_i),
    \end{aligned}
\vspace{-1mm}
\end{equation}
where $\text{Dir}$ stands for Dirichlet distribution, $B$ is a draw from the Beta process which provides a set of weights for the potentially infinite number of hidden states. For each time series $i$, an $X_i$ is drawn from a Bernoulli process parameterized by $B$. Each $X_i$ can be used to construct a binary vector $f_i$ indicating which of the global hidden state are selected in the $i^{th}$ time series. Then the transition probability $\pi^{i}_j$ of the AR-HMM is drawn from a Dirichlet distribution with self-transition bias $\kappa$ for each state $j$. 


The generative model is a total Bayesian model which implies that the model can be inferred from data according to the Bayes' rule. The parameters such as the hidden variable $z^i_t$ and $\theta_{z_t^i}$ can be learned from data using the Markov chain Monte Carlo (MCMC) method \cite{fox2014joint}. 

\begin{example}
A driver and hardware-in-the-loop simulation system is used as an example to validate the proposed approach. Markers are put on the left/right hand of the driver and the steering wheel, a time series of positions of these markers are collected using the Optitrack system. An example of the raw data is shown in Figure \ref{fig:rawdata} which consists of driving motions such as turning left and right. Using the Bayesian non-parametric learning method, motion features can be identified automatically. Figure \ref{fig:compare_whole} gives a comparison of learning results using HMM and AR-HMM as generative models for the same training data. Different motion features are labeled by different colors. From the result, there are only $3$ motion features detected using the HMM generative model while there are $53$ motion features detected using the AR-HMM generative model, which is much more than that of the HMM assumption. To give a more detailed comparison, data points from $1$ to $1000$ are zoomed out which is shown in Figure \ref{fig:detail}. From the results, some dynamic motions are not identified by the HMM model while they are detected using the AR-HMM. The reason behind this phenomenon is that the AR-HMM use a dynamic system to model the observed data and considers the correlation among observations while HMM assumes observation independence.

\begin{figure}[!t]
\centering
\includegraphics[width=0.9\linewidth]{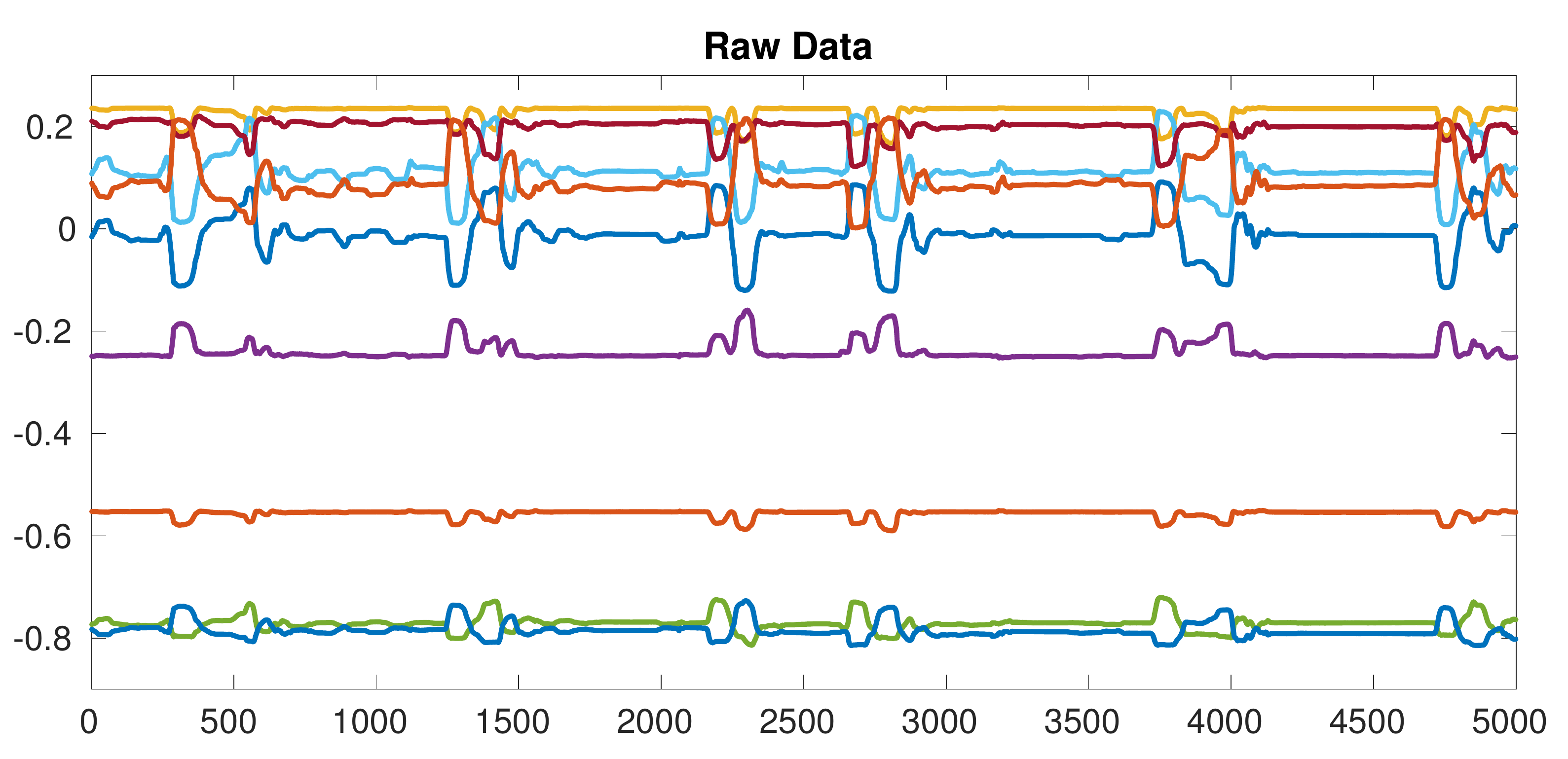}
\vspace{-3mm}
\caption{An example of the raw data collected from a driving experiment.}
\label{fig:rawdata}
\vspace{-6mm}
\end{figure}

\begin{figure}[!t]
\centering
\includegraphics[width=0.9\linewidth]{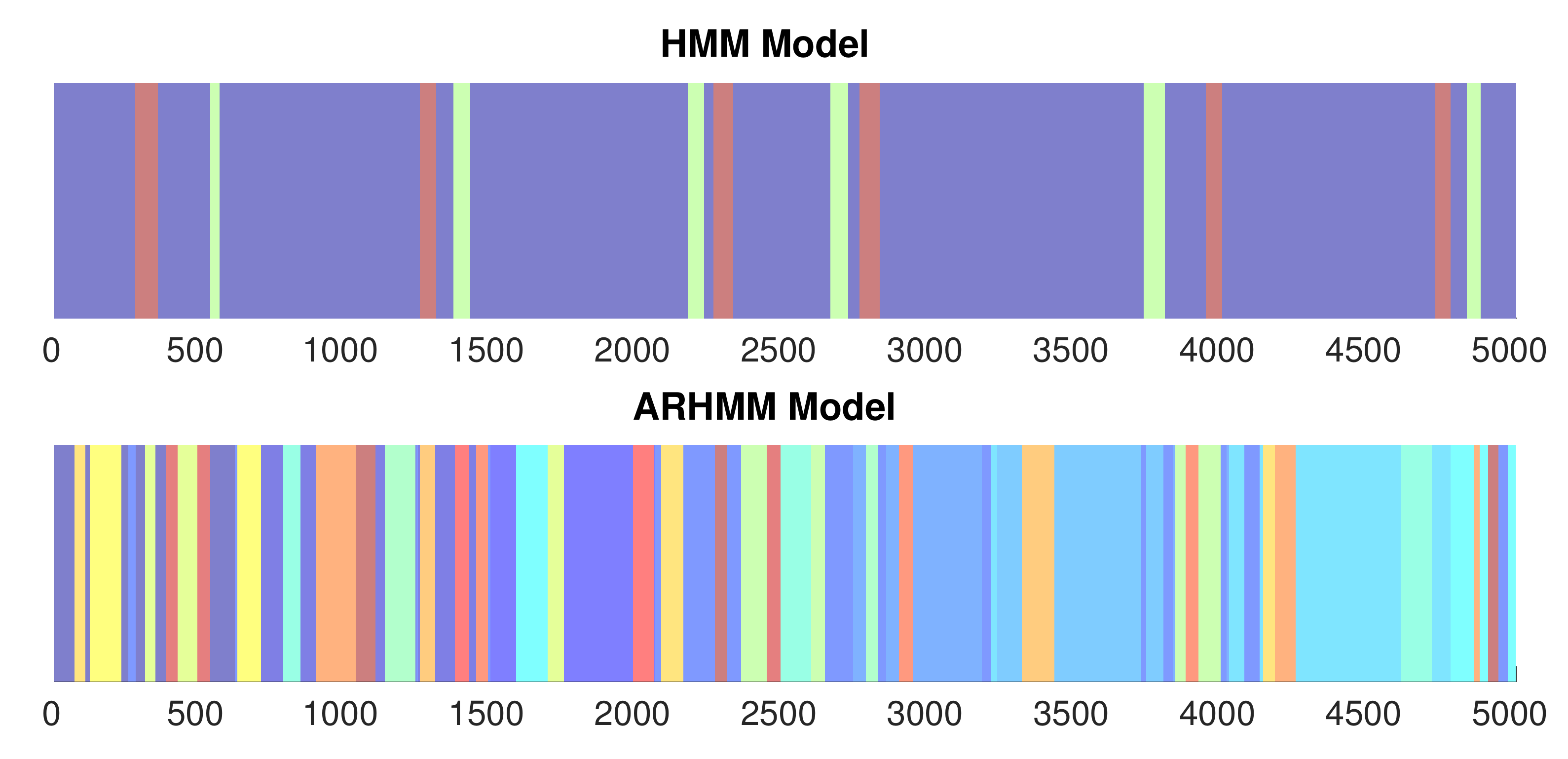}
\vspace{-3mm}
\caption{The learning results of using the Bayesian non-parametric method on human motion trajectories. }
\label{fig:compare_whole}
\vspace{-4mm}
\end{figure}

\begin{figure}[!t]
\centering
\includegraphics[width=0.9\linewidth]{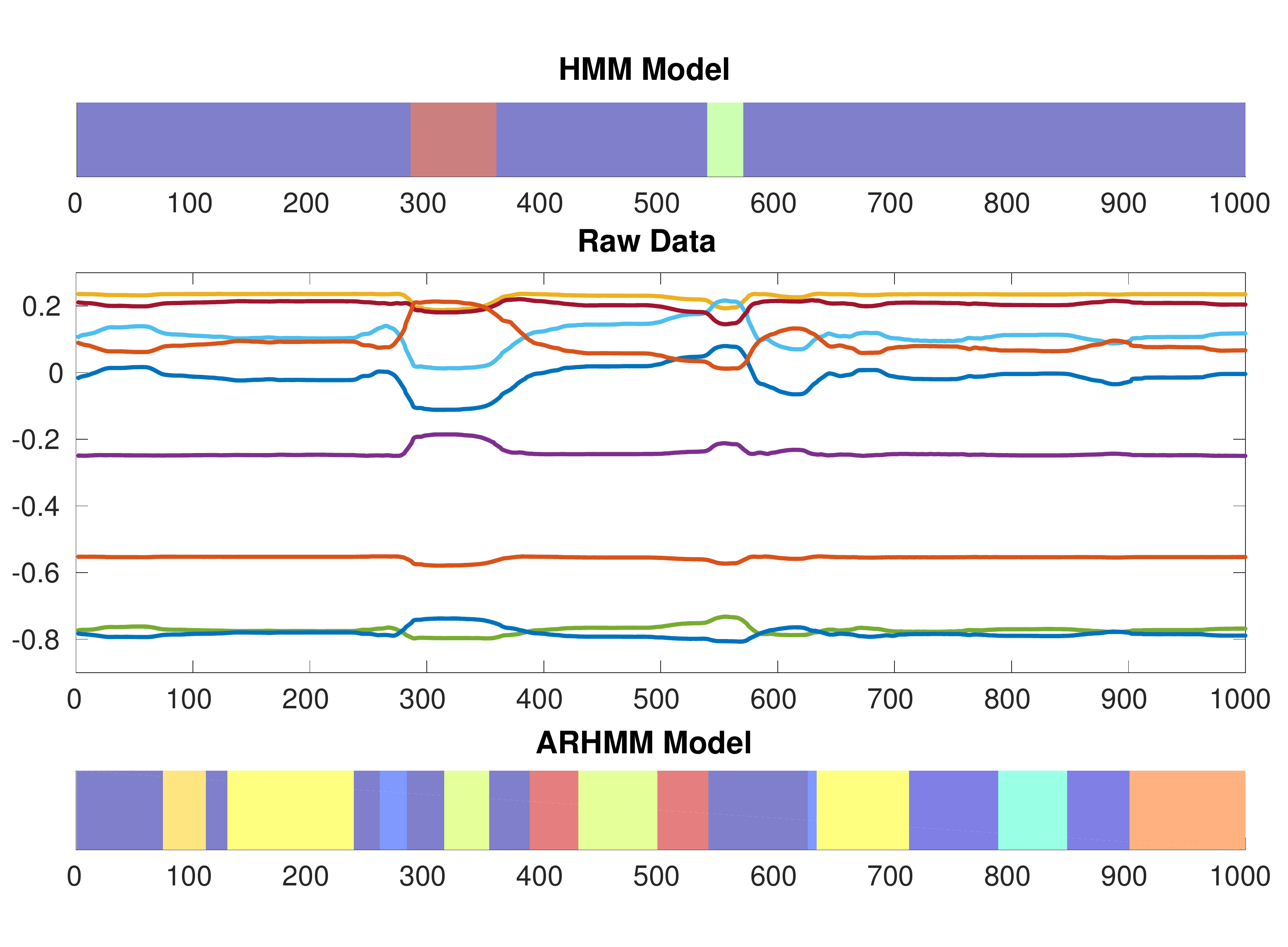}
\vspace{-3mm}
\caption{A detail comparison of using the HMM generative model and the AR-HMM generative model.}
\label{fig:detail}
\vspace{-6mm}
\end{figure}
\end{example}

\subsection{Construct VAR-POMDP model}
\label{sec:constructionPOMDP}
Based on the features identified in section \ref{sec:MotionFE}, the VAR-POMDP model can be directly constructed. First, the state space $S$ of the VAR-POMDP model is defined as the product of the state space of human, robots and the environment. The state space of the human can be defined as the union of motion feature $\theta_i$ identified in section \ref{sec:MotionFE}. Each state can be labeled manually with a physical meaning. According to the physical meaning, the labeling function can be defined. For example, in a driving scenario, the label could be $'safe','danger'$ or $'failure'$. The observation space $O$ is defined to be the $n$-dimensional vector space $\mathbb{R}^n$ which could be continuous sensor readings. The observation function $E$ is defined as a multivariate Gaussian distribution, 
\begin{equation*}
    E(o_t|\mathbf{o_{t-1}},s_t) \sim \mathcal{N}(\sum_{j=1}^{r} A_{j,s_t} o_{t-j}, \Sigma_{s_t}),
\vspace{-1mm}
\end{equation*}
where $\mathbf{o_{t-1}}$ is the observation history. 

After identifying the state space and observation distribution, our next step is to learn the transition probability. To learn the exact transition probability is difficult due to reasons such as limited data. In this case, modeling uncertainties will make the learned transition probabilities subject to a certain confidence level which motivates us to apply the Chernoff bound to reason the accuracy of the transition probabilities for VAR-POMDP. Details of the transitions probability learning can be found in \cite{zheng2018pomdp}. From the Chernoff bound, the estimation error of the transition probability can be sufficiently small with high confidence as long as the training data is sufficient enough.

\begin{example}
In this example, parts of the model are shown in Figure \ref{fig:mmexam1} to illustrate the model construction process. For the driver assistance system, actions are designed to increase car safety and road safety by providing real-time instructions, warnings or directly controlling the vehicles. For example, one action can be designed as instruction $'No \ right \ turn'$ and another one can be designed to increase the steering torque. Different actions have different influence on the human driver and therefore have different transition probability. 
\begin{figure}
\centering
    \begin{tikzpicture}[->,>=stealth', node distance=1.5cm]
        \node[state] (q0) {$s_0$};
        \node[state] (q1) [below of = q0] {$s_1$};
        \node[state] (q2) [right of = q1] {$s_2$};
        \path[->] (q0) edge              node [left] {$[a_1,a_2]$} (q1)
                  (q0) edge              node  {} (q2)
                  
                  (q1) edge [loop left]  node  {} (q1)
                  (q1) edge              node  {} (q2)
                  (q2) edge [loop right] node  {} (q2);
    \end{tikzpicture}
    \vspace{-3mm}
    \caption{An example of the VAR-POMDP for the driver assistance system. }   
    \label{fig:mmexam1}
    \vspace{-6mm}
\end{figure}
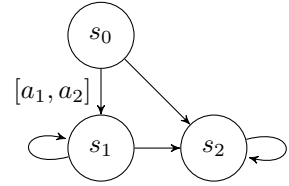
Each state is a composition of environment status and human intentions. State $s_0$ represents $(on\_lane, turn\_right)$, $s_1$ represents $(on\_lane, Normal\_driving)$ and $s_2$ represents $(off\_lane, \times)$ where $\times$ can be any human intentions. Based on observations, a belief is maintained over hidden states according to the Bayes' rule and an action can be selected to maximize the safety of the system. 
\end{example}

\section{VAR-POMDP Planning}\label{sec:problem_formulation}
In the previous section, we showed that the proposed  VAR-POMDP model can be effectively learned from data. This section aims to illustrate that the planning problem based on the proposed VAR-POMDP model is also tractable. Particularly, we consider PCTL specification as safety and reachability requirement and apply PBVI to solve the PCTL model checking problem.

\subsection{
Converting the PCTL model checking into a dynamical programming problem}\label{sec:reward}

We first show that the bound until specification model checking problem described in Problem \ref{pro:problem1} can be converted into a dynamic programming problem. The PCTL model checking problem is well studied for MDP model \cite{rutten2004mathematical}, we generalize the result to the VAR-POMDP model. The state space of the model can be divided into three disjoint subsets. 
\begin{align*}
S^{yes} &=  Sat(\phi_2)\\
S^{no}  &=   S \setminus (Sat(\phi_1) \cup Sat(\phi_2))\\
S^?     &=   S \setminus (S^{yes} \cup S^{no}).
\vspace{-2mm}
\end{align*}
All states in $S^{yes}$ satisfy $\phi_2$ and all states in $S^{no}$ dissatisfy $\phi_1$ and $\phi_2$. The state in $S^?$ satisfies $\phi_1$ but not $\phi_2$. Once the system runs to states in $S^{no}$, the satisfaction probability will be zero no matter where it goes in the future. And once the system runs to states in $S^{yes}$, the satisfaction probability dependents only on the running prefix. Thus changing the states in $S^{no}$ and $S^{yes}$ absorbed and assigning reward $1$ for states in $S^{yes}$ and $0$ for other states does not change the satisfaction probability. The maximum satisfaction probability can be solved recursively by value iteration, 
\begin{equation}
\label{eq:iteration}
    V_{t+1}(b,\mathbf{o})=\max_{a \in A} [\sum_{s,s'} b_s  \int_{o'} T_{s'}^{s,a} E_{o'}^{s',\mathbf{o}} V_{t}(b',\mathbf{o'}) {d}o'],
    \vspace{-2mm}
\end{equation}
where $b_s$ stands for the belief on state $s$, $T_{s'}^{s,a}$ and $E_{o'}^{s',o}$ are notations for transition probability and observation distribution, $b'$ is the posterior belief after observing $o'$ and $\mathbf{o}$ represents observation history. The value $V_t(b,\mathbf{o})=p_b^{max}(\phi_1 \cup^{\leq t} \phi_2)$ is the maximum probability that satisfies the specification when the belief is $b$ and observation history is $\mathbf{o}$. The initial condition $p_b^{max}(\phi_1 \cup^{\leq 0} \phi_2) = \sum_s b_s p^0_s$ where $p^0_s = 1$ if $s \in S^{yes}$ and $p_s^0 = 0$ otherwise.

\subsection{Point-based Value Iteration for VAR-POMDP}
{The main challenge arising from the VAR-POMDP model is the curse of dimensionality and the curse of history.} In equation (\ref{eq:iteration}), it is impossible to enumerate observations $o'$ since $o'$ is in continuous space. The value function is not only a function of belief $b$ but also a function of observation history $\mathbf{o}$. Thus the exact dynamic programming approach cannot be applied to solve the problem \cite{sondik1971optimal}. Inspired by the dynamic discretization approach \cite{hoey2005solving}, we propose to use the PBVI to solve the dynamic programming problem. 

In PBVI, a set of belief points $\{\tilde b^1,...,\tilde b^M\}$ is selected and the value function is updated only on these belief points. Thus PBVI gives an approximate solution where the approximation error dependents on the belief point selection. 

\begin{theorem}
The optimization problem defined recursively by equation (\ref{eq:iteration}) can be solved using PBVI algorithm on a predefined belief set $\{\tilde b^1,...,\tilde b^M\}$. At these belief points, the value function $V_t(b,\mathbf{o})$ can be approximated by a piece-wise linear and convex function, which is written as
\begin{equation}
\label{eq:approlinear}
    V_t(b,\mathbf{o}) = \max_{k} \sum_{s} b_s \alpha_s^{t,k}, 
    \vspace{-1mm}
\end{equation}
for a set of $\alpha$-vector $\{\alpha^{t,1}, ..., \alpha^{t,M}\}$ where $\alpha_s^{t,k}$ represents the $s^{th}$ element of the $k^{th}$ $\alpha$-vector for the $t$-step-to-go value function. 
\end{theorem}
\begin{proof}
The theorem is proved by induction. Assume that the value function can be expressed as $V_t(b,\mathbf{o}) = \max_{k} \sum_{s} b_s \alpha_s^{t,k}$. When $t=0$, the initial $\alpha$-vector $\alpha^{0,1}$ is defined corresponding to the reward of the reaching state. According to the $0$-step-to-go value function defined in section \ref{sec:reward}, the initial $\alpha$-vector $\alpha^{0,1} = p^0$ and it is a constant vector. When $t=1$, substitute the updated belief 
\begin{equation}
    b'_{s'} = \frac{E^{s',\mathbf{o}}_{o'} \sum_s T^{s,a}_{s'} b_s}{\sum_{s'} E^{s',\mathbf{o}}_{o'} \sum_s T^{s,a}_{s'} b_s},
\end{equation}
and value function (\ref{eq:approlinear}) into the right hand of equation (\ref{eq:iteration}),  
\begin{equation}
\label{eq:step1}
    \begin{aligned}
    V_1(b,\mathbf{o}) = & \max_{a \in A} [ \sum_s b_s \int_{o'} \sum_{s'} T_{s'}^{s,a} E_{o'}^{s',\mathbf{o}} \alpha^{0,1}_{s'}{d}o']\\
     = &  \max_{a \in A} [\sum_s b_s \sum_{s'} T_{s'}^{s,a} \alpha^{0,1}_{s'}].
    \end{aligned}
    \vspace{-2mm}
\end{equation}
For each point in the belief set $\{\tilde b^1,...,\tilde b^M\}$, we select the action $a$ that maximizes the expected reward. Then a set of $\alpha$-vector $\{\alpha^{1,k}\} , k \in \{1,...,M\}$ is updated. 

Assume for all $t \geq 2$, the value function can be expressed as $V_{t-1}(b,\mathbf{o}) = \max_{k} \sum_{s} b_s \alpha_s^{t-1,k}$. 
Then the $t$-step-to-go value function can be expressed as 
\begin{equation}
V_t(b,\mathbf{o}) = \max_{a \in A} [ \int_{o'} \max_{k} \sum_{s,s'} b_s T_{s'}^{s,a} E_{o'}^{s',\mathbf{o}} \alpha^{t-1,k}_{s'}{d}o'].
\end{equation}
Due to the $\max$ function inside the integral, directly calculating the integration is not tractable. Inspired by the work of \cite{hoey2005solving}, we break the observation space into sub-space and use the sampling-based approach to approximate the integration. 

Let $z_k(\mathbf{o})$ be the sub-space of the observation $o'$ that makes the expected reward maximum for a specific belief point and $\alpha$-vector. For a given belief $\tilde b^i$  and $\alpha$-vector $\alpha^{t-1,k}$, 
\begin{equation}
    z_k(\mathbf{o}) = \{o'| \max_{l} [\sum_s \tilde b^i_s \sum_{s'} T_{s'}^{s,a} E_{o'}^{s',\mathbf{o}} \alpha^{t-1,l}_{s'}] = k \}.
    \vspace{-1mm}
\end{equation}
Then the value iteration is converted to 
\begin{equation}
\label{eq:iteration10}
    \begin{aligned}
    V_t(b,\mathbf{o}) = & \max_{a \in A} [\sum_{k} \sum_{s,s'} b_s T_{s'}^{s,a} \alpha_{s'}^{t-1,k} \int_{z_k(\mathbf{o})} E_{z'}^{s',\mathbf{o}} {d}z' ]\\
    = &\max_{a \in A} [ \sum_{k} \sum_{s,s'} b_s T_{s'}^{s,a} \alpha_{s'}^{t-1,k} \text{Pr} (z_k(\mathbf{o})|s')],
    \end{aligned}
    \vspace{-2mm}
\end{equation}
where $\text{Pr} (z_k(\mathbf{o})|s') =  \int_{z_k(\mathbf{o})} E_{z'}^{s',\mathbf{o}} {d}z'$. To calculate the integral directly is not tractable, thus sampling approach is used to approximate the integration. 

The sub-space $z_k(\mathbf{o})$ currently is a function of observation history $\mathbf{o}$ which causes the sampling from $E_{o'}^{s',o}$ intractable. However, it can be shown that the probability $\text{Pr} (z_k(\mathbf{o})|s')$ does not dependent on observation history $\mathbf{o}$. 

Let $\hat o = o'-\sum_{j=1}^{r} A_{j,s'} o_{t-j}$ be a new variable for which the observation distribution is $E^{s'}_{\hat o} \sim \mathcal{N}(0, \Sigma_{s'})$. Then 
\begin{equation}
    \text{Pr} (z_k(\mathbf{o})|s') =  \int_{z_k(\mathbf{o})} E_{z'}^{s',\mathbf{o}} {d}z' =  \int_{z_k} E_{\hat z}^{s'} {d} \hat z.
\end{equation}
where $z_k = \{ \hat o| \max_{l} [\sum_s \tilde b^i_s \sum_{s'} T_{s'}^{s,a} E_{\hat o}^{s'} \alpha^{t-1,l}_{s'}] = k \}$. The the probability $\text{Pr} (z_k(\mathbf{o})|s')$ is no longer a function of observation history and can be approximated using the Monte Carlo method. For each state $s'$, we sample $L$ observations $\hat o$ from $E_{\hat o}^{s'}$ and approximate the integration by 
\begin{equation}
    \text{Pr} (z_k|s') = \frac{l_{mc}^k}{L},
\end{equation}
where $l_{mc}^k$ is the number of samples that fall into the sub-region $z_k$. Back to the value iteration equation (\ref{eq:iteration10}), the $t$-step-to-go $\alpha$-vector $\alpha^{t,k}$ is not a function of $\mathbf{o}$. Thus the value function $V_t(b,\mathbf{o}) = \max_{k'} \sum_{s} b_s \alpha_s^{t,k'}$ where 
\begin{equation}
\label{eq:alphaupdate}
    \alpha_s^{t,k'} = \max_{a \in A} [ \sum_{k} \sum_{s'} T_{s'}^{s,a} \alpha_{s'}^{t-1,k} \text{Pr} (z_k(\mathbf{o})|s')]. 
\end{equation}
Then $\max$ function of equation (\ref{eq:alphaupdate}) is evaluated on belief set $\{\tilde b^1,...,\tilde b^M\}$. By induction, we conclude that the value function $V_t(b,\mathbf{o})$ can be approximated by a piece-wise linear and convex function for all $t \geq 1$. 
\end{proof}

{From the proof, it is shown that although the correlation of observations is considered, the piece-wise linear function used to approximate the value function $V_t(b,\mathbf{o})$ is not a function of observation history $\mathbf{o}$. At each time step, $M|A|$ $\alpha$-vectors are created while the number of final $\alpha$-vectors stored is limited to $M$ (in time complexity $M^2 |S||A|$). Thus, the whole value iteration takes polynomial time and the size of $\alpha$-vectors remains constant. The PBVI is an approximation algorithm, it scarifies the accuracy of the solution to achieve efficiency. 
The approximation error depends on how densely the belief set $\{\tilde b^1,...,\tilde b^M\}$ samples from the belief simplex $\Delta$. Following the proof of Lemma $1$ in \cite{pineau2003point}, it can be shown that the approximation error is bounded by the density of the belief set. 

\begin{assertion}
The error induced by the point-based value iteration algorithm for the VAR-POMDP model is bounded by the density $\epsilon_B$ of the belief set $B = \{\tilde b^1,...,\tilde b^M\}$, which is defined as $\epsilon_B = \max_{b' \in \Delta} \min_{b \in B} ||b - b'||_1$.  
\end{assertion}

\begin{proof}
Let $\tilde b'$ be the point where point-based value iteration makes the pruning error worst and let $\tilde b^i$ be the closest belief to $\tilde b'$. Let $\alpha'$ be the vector maximal at $\tilde b'$ and $\alpha^{H,j}$ be the vector maximal at $\tilde b^i$. Then it is easy to get $ \alpha' \tilde b^i \leq \alpha^{H,j} \tilde b^i$. 
The pruning error $e$ for the $H$-step-to-go value function is bounded by
\begin{equation}
    \begin{aligned}
        e  & \leq \alpha' \tilde b' - \alpha^{H,j} \tilde b' \\
        & = \alpha' \tilde b' - \alpha^{H,j} \tilde b' + \alpha' \tilde b^i - \alpha' \tilde b^i \\
        & \leq \alpha' \tilde b' - \alpha^{H,j} \tilde b' + \alpha^{H,j} \tilde b^i - \alpha' \tilde b^i\\
        & = (\alpha' - \alpha^{H,j})( \tilde b' - \tilde b^i )\\
        & \leq ||\alpha' - \alpha^{H,j}||_\infty || \tilde b' - \tilde b^i ||_1\\
        & \leq \epsilon_B
    \end{aligned}
\end{equation}
The last inequality holds because the $\alpha$-vector represents the achievable reward which is bounded by $1$ in our case. 
\end{proof}
Intuitively, the proof shows that the more densely the belief set $\{\tilde b^1,...,\tilde b^M\}$ is selected the smaller the approximation error will be. 

The value iteration algorithm is summarized in Algorithm $1$. Since it is a finite-step value iteration, the algorithm will always converge.}

\begin{algorithm}
\label{alg:pbvi}
\caption{PBVI Algorithm for VAR-POMDP}
    \begin{algorithmic}[1]
    \renewcommand{\algorithmicrequire}{\textbf{Input:}}
    \renewcommand{\algorithmicensure}{\textbf{Output:}}
    \REQUIRE VAR-POMDP model $P$, Belief Points $\{\tilde b^1,...,\tilde b^M\}$. 
    \ENSURE  Vector set $\{\alpha^{i,1},...,\alpha^{i,M}\},i \in \{1,...,H\}$.
    \STATE Initialize $\alpha$-vector $\alpha^{0,1} = p^0$. 
    \WHILE{$t \leq H$}
        \IF{$t=1$}
        \FOR{$\tilde b^i \in \{\tilde b^1,...,\tilde b^M\}$}
            \STATE Update $\alpha$-vector according to equation (\ref{eq:step1}).
        \ENDFOR
        \ELSE        
        \FOR{$\tilde b^i \in \{\tilde b^1,...,\tilde b^M\}$}
            \STATE Approximate $\text{Pr} (z_k|s')$;
            \STATE Update $\alpha$-vector according to equation (\ref{eq:alphaupdate}).
        \ENDFOR            
        \ENDIF
        \STATE $t = t+1$.
    \ENDWHILE
    \end{algorithmic}
\end{algorithm}

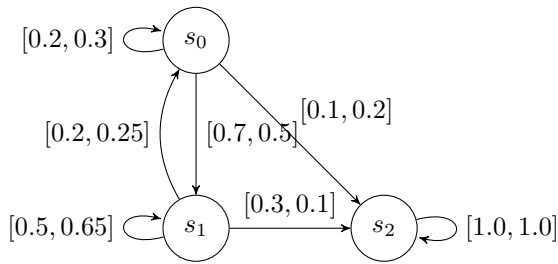
\begin{figure}
\centering
    \begin{tikzpicture}[->,>=stealth',node distance=2.5cm]
        \node[state] (q0) {$s_0$};
        \node[state] (q1) [below of = q0] {$s_1$};
        \node[state] (q2) [right of = q1] {$s_2$};
        \path[->] (q0) edge [loop left]  node         {$[0.2,0.3]$} (q0)
                  (q0) edge              node [right]  {$[0.7,0.5]$} (q1)
                  (q0) edge              node [above right] {$[0.1,0.2]$} (q2)
                  
                  (q1) edge [bend  left]  node [left] {$[0.2,0.25]$} (q0)           
                  (q1) edge [loop left]  node [left]  {$[0.5,0.65]$} (q1)
                  (q1) edge              node [above] {$[0.3,0.1]$} (q2)
                  
                  (q2) edge [loop right] node          {$[1.0,1.0]$} (q2);
    \end{tikzpicture}
    \vspace{-3mm}
    \caption{The transition relation of a three-state VAR-POMDP model.}
    \label{fig:exam2}
    \vspace{-6mm}
\end{figure}

\begin{example}
A three-state VAR-POMDP model shown in Figure \ref{fig:exam2} is used to validate the PBVI algorithm. The transition probability of action $a_1$ and $a_2$ are shown in the figure. Each state represents a motion feature of a three-dimensional motion trajectory and state $s_2$ are labeled as $'Fail'$. The belief points are selected to be $[1,0,0],[0,1,0],[0.5,0.5,0],[0.6,0.3,0.1],[0.3,0.4,0.3]$. The PCTL specification is given as 
$P_{\leq 0.5} [\phi_1 \cup^{\leq 4} \phi_2]$ with $\phi_1 = true$ and $\phi_2 =$ $'Fail'$. Using the PBVI, five $\alpha$-vectors are solved which are $[0.67,0.63,1]$,$[0.63,0.71,1]$, $[0.63,0.71,1]$,$[0.64,0.71,1]$,$[0.64,0.71,1]$. If the initial belief is $[1,0,0]$, then the satisfaction probability is $0.67$. Since the upper bound is $0.5$, the specification is not satisfied. 
\end{example}

\section{Conclusion}\label{sec:conc}

In this paper, we proposed the VAR-POMDP model for HRC which is an extension of the traditional POMDP model by considering the correlation among observations. We showed the tractability of the proposed model by providing a learning framework and a planning algorithm. In the learning framework, we proposed to use Bayesian non-parametric methods to learn the VAR-POMDP model from demonstrations effectively. We proved that the PBVI algorithm can be extended on the VAR-POMDP model to solve a model checking problem for bounded until specification in PCTL. In both the learning and planning process, approximations were used to estimate parameters of the model including the transition probability, observation distribution and the potential belief points. Evaluating the influence of these approximations on system performance will be future work.  

\bibliographystyle{unsrt}
\bibliography{ARPOMDP}
\end{document}